\documentclass{article}

\usepackage{arxiv}

\usepackage[utf8]{inputenc} 
\usepackage[T1]{fontenc}    
\usepackage{hyperref}       
\usepackage{url}            
\usepackage{booktabs}       
\usepackage{amsfonts}       
\usepackage{nicefrac}       
\usepackage{microtype}      
\usepackage{lipsum}
\usepackage{graphicx}
\graphicspath{ {./images/} }
\usepackage{amssymb}
\usepackage{amsmath}

\usepackage{subcaption}
\usepackage{xcolor}
\usepackage{tabularx}
\usepackage[ruled,linesnumbered]{algorithm2e}
\usepackage{amsthm}
\newtheorem{lemma}{Lemma}[section]       
\theoremstyle{definition}                
\usepackage{nomencl}
\makenomenclature
\usepackage{etoolbox}
\renewcommand\nomgroup[1]{%
	\item[\bfseries
	\ifstrequal{#1}{R}{Roman Symbols}{%
		\ifstrequal{#1}{G}{Greek Symbols}{%
			\ifstrequal{#1}{A}{Abbreviations}{%
				\ifstrequal{#1}{S}{Subscripts}}}}%
	]}

\title{Analytical Extraction of Conditional Sobol' Indices via Basis Decomposition of Polynomial Chaos Expansions}

\author{
	Shijie Zhong \\
	School of Power and Energy\\
	Northwestern Polytechnical University\\
	Xi'an, Shanxi 710129 \\
	\texttt{zhongsj@mail.nwpu.edu.cn} \\
	\And
	Jiangfeng Fu \\
	School of Power and Energy\\
	Northwestern Polytechnical University\\
	Xi'an, Shanxi 710129 \\
	\texttt{fjf@nwpu.edu.cn} \\
}

\begin{document}
\maketitle
\begin{abstract}
In uncertainty quantification, evaluating sensitivity measures under specific conditions (i.e., conditional Sobol' indices) is essential for systems with parameterized responses, such as spatial fields or varying operating conditions. Traditional approaches often rely on point-wise modeling, which is computationally expensive and may lack consistency across the parameter space. This paper demonstrates that for a pre-trained global Polynomial Chaos Expansion (PCE) model, the analytical conditional Sobol' indices are inherently embedded within its basis functions. By leveraging the tensor-product property of PCE bases, we reformulate the global expansion into a set of analytical coefficient fields that depend on the conditioning variables. Based on the preservation of orthogonality under conditional probability measures, we derive closed-form expressions for conditional variances and Sobol' indices. This framework bypasses the need for repetitive modeling or additional sampling, transforming conditional sensitivity analysis into a purely algebraic post-processing step. Numerical benchmarks indicate that the proposed method ensures physical coherence and offers superior numerical robustness and computational efficiency compared to conventional point-wise approaches.
\end{abstract}


\section{Introduction}
Polynomial Chaos Expansion (PCE) has emerged as a cornerstone in uncertainty quantification (UQ) due to its rigorous mathematical foundation and computational efficiency. The defining advantage of PCE is its "analytical transparency": once the expansion coefficients are determined, the global statistical moments and variance-based Sobol' indices can be derived directly through algebraic operations on these coefficients, bypassing the need for expensive Monte Carlo sampling \cite{sudret2008global}. This spectral approach, rooted in the Wiener-Askey scheme \cite{Xiu2002}, exhibits superior convergence properties for model responses with sufficient regularity. In recent years, the introduction of sparse representation frameworks, ranging from Bayesian Compressive Sensing (BCS) \cite{ji2008bayesian, karimi2025bi} to greedy algorithms such as Orthogonal Matching Pursuit (OMP) \cite{tropp2007signal, Rezaiifar1995} and Least Angle Regression (LAR) \cite{Blatman2011AdaptiveSP}, has significantly enhanced its applicability to high-dimensional problems in complex engineering systems.

Building upon these theoretical strengths, PCE has been widely adopted to evaluate global performance metrics in various domains, such as rotating machinery \cite{8814745} and centrifugal pumps \cite{fracassi2022shape}. However, in many advanced applications, one is often interested not only in global scalars but also in the conditional behavior of the system, that is, how sensitivities evolve when a subset of input variables (conditioning variables) is fixed at specific values. A typical example is the analysis of spatially distributed stochastic fields, which are common in stochastic CFD problems \cite{SALEHI2017296, SALEHI2018183}. To date, the most prevalent strategy for such problems is point-wise modeling, where independent PCE surrogates are constructed at each discrete spatial location or condition. This framework has been applied to investigate uncertainties in fluid properties and geometric variations in blood pumps \cite{salehi2018flow, karimi2021stochastic}, providing insights into the robustness of hydraulic performance and hemocompatibility.

While straightforward, the point-wise approach is theoretically fragmented and computationally intensive \cite{migliorati2013approximation, sudret2007uncertainty}. It treats intrinsically coupled conditions as isolated black boxes, which often leads to a lack of analytical continuity and potential non-physical oscillations in the resulting sensitivity fields.

This paper addresses this gap by revealing that conditional Sobol' indices are inherently encoded within the basis functions of a single global PCE. We demonstrate that by exploiting the tensor-product nature of the PCE basis, a pre-trained global expansion can be analytically "projected" onto any conditional subspace. This process yields a set of analytical conditional coefficient fields, which are closed-form functions of the conditioning variables. Consequently, the conditional Sobol' indices are no longer estimated through repetitive local modeling but are extracted directly from the global coefficients via pure algebraic manipulation. This framework ensures both physical coherence across the entire conditional domain and a near-zero computational cost for post-processing, offering a more elegant and robust theoretical solution for conditional sensitivity analysis.

\section{Theoretical foundation of PCE}
To establish the analytical framework for conditional sensitivity analysis, the key components of the classical Polynomial Chaos Expansion (PCE), including its spectral representation, orthonormal basis construction, and regression-based solution, are briefly reviewed.

\subsection{PCE representation of the model response}
Consider an $M$-dimensional random vector $\boldsymbol{X}$, whose components are mutually independent and have a joint probability density function $f_{\boldsymbol{X}}$. Let the computational model $\mathcal{M}(\cdot)$ map the random vector of the input to the random variable of the output $Y = \mathcal{M}(\boldsymbol{X})$. Assuming $Y$ has a finite variance, based on PCE, $Y$ can be represented as:

\begin{equation}
	\label{PCE}
	Y=\mathcal{M} \left( \boldsymbol{X} \right) =\sum_{\boldsymbol{\alpha }\in \mathbb{N} ^M}{y_{\boldsymbol{\alpha }}\Psi _{\boldsymbol{\alpha }}\left( \boldsymbol{X} \right)}
\end{equation}
where $\boldsymbol{\alpha}$ is the multi-index, $y_{\boldsymbol{\alpha}}$ are the corresponding expansion coefficients, and $\Psi_{\boldsymbol{\alpha}}(\boldsymbol{X})$ are orthonormal multivariate polynomials with respect to $f_{\boldsymbol X}$. These polynomials are typically constructed by the tensor product of one-dimensional orthonormal polynomials:
\begin{equation}
	\Psi_{\boldsymbol{\alpha}}(\boldsymbol{X}) = \prod_{i=1}^M \phi_{\alpha_i}(x_i)
\end{equation}
which satisfy the orthonormality condition $\left\langle \Psi_{\boldsymbol{\alpha}}, \Psi_{\boldsymbol{\beta}} \right\rangle = \delta_{\boldsymbol{\alpha} \boldsymbol{\beta}}$, where $\delta_{\boldsymbol{\alpha} \boldsymbol{\beta}}$ is the multi-dimensional Kronecker symbol.

\begin{table}[htbp]
	\centering
	\caption{Correspondence between orthogonal polynomials and probability distributions}
		\begin{tabularx}{\textwidth}{
				>{\hsize=0.5\hsize}X
				>{\hsize=0.6\hsize}X
				>{\hsize=0.9\hsize}X}
			\hline
			\textbf{Polynomial} & \textbf{Distribution} & \textbf{Normalization Factor} $\gamma_k$\\
			\hline
			Hermite $H_k(x)$ & Gaussian $(-\infty,\infty)$ & $k!$ \\
			Laguerre $L_k^{(\alpha)}(x)$ & Gamma $[0,\infty)$ & ${\Gamma(k+\alpha+1)}\big/\left({k! \Gamma(\alpha+1)}\right)$ \\
			Legendre $P_k(x)$ & Uniform $[-1,1]$ & ${1}\big/\left({2k+1}\right)$ \\
			\hline
	\end{tabularx}
	\label{pce_mapping_consistent}
\end{table}

For different types of input variables, the one-dimensional polynomials basis $\phi_k(x)$ is constructed based on the corresponding probability density function \cite{xiu2002wiener}, as shown in Table \ref{pce_mapping_consistent}. Taking the uniform distribution as an example, these basis functions are defined as:
\begin{equation}
	\phi_k(x) = \frac{P_k(x)}{\sqrt{\gamma_k}}
\end{equation}
where $P_k(x)$ is the standard orthogonal Legendre polynomial and $\gamma_k$ is the normalization factor, satisfying $\left\langle \phi_i(x), \phi_j(x) \right\rangle = \delta_{ij}$.
	
Furthermore, The standard orthogonal Hermite, Laguerre, and Legendre polynomials can be calculated using the following recurrence formulas, respectively:
	
\begin{equation}
	H_{k+1}(x) = 2xH_k(x) - 2kH_{k-1}(x)
\end{equation}
\begin{equation}
	(k+1)L_{k+1}^{(\alpha)}(x) = (2k+\alpha+1-x)L_k^{(\alpha)}(x) -(k+\alpha)L_{k-1}^{(\alpha)}(x)
\end{equation}
\begin{equation}
	(k+1)P_{k+1}(x) = (2k+1)xP_k(x) - kP_{k-1}(x)
\end{equation}

In practical applications, for computational convenience, the PCE is usually limited such that the total polynomial order does not exceed $p$:
\begin{equation}
	\mathcal{A} ^{M,p}=\left\{ \boldsymbol{\alpha }\in \mathbb{N} ^M:\left| \boldsymbol{\alpha } \right|\leqslant p \right\}
\end{equation}

The number of terms, $P$, in the truncated series is:
\begin{equation}
	\label{Card}
	P=\mathrm{card}\mathcal{A} ^{M,p}=\left( \begin{array}{c}
		M+p\\
		p\\
	\end{array} \right) 
\end{equation}

The truncated PCE is then expressed as:
\begin{equation}
	\label{tunPCE}
	\mathcal{M} ^{PC}\left( \boldsymbol{X} \right) =\sum_{\boldsymbol{\alpha }\in \mathcal{A} ^{M,p}}{y_{\boldsymbol{\alpha }}\Psi _{\boldsymbol{\alpha }}\left( \boldsymbol{X} \right)}
\end{equation}

\subsection{Coefficient calculation for the PCE model}
For a given polynomial basis, the PCE coefficients $y_{\boldsymbol{\alpha}}$ can be obtained through various non-intrusive methods. The regression method is adopt here as it does not depend on specific quadrature points, offering greater flexibility in model calls \cite{migliorati2013approximation}.

To construct the regression form, the infinite series in Eq.\ref{PCE} is first truncated into a finite-dimensional form, which can be written as:
\begin{equation}
	Y=\mathcal{M} \left( \boldsymbol{X} \right) =\sum_{j=0}^{P-1}{y_j\Psi _j\left( \boldsymbol{X} \right)}+\varepsilon _P\equiv \boldsymbol{y}^{\top}\Psi \left( \boldsymbol{X} \right) +\varepsilon _P
	\label{tPCE}
\end{equation}
where $\varepsilon_P$ is the truncation error; $\boldsymbol{y} = (y_0,\cdots,y_{P-1})$ is the coefficient vector and ${\Psi}(\boldsymbol{X}) = (\Psi_0(\boldsymbol{X}),\cdots,\Psi_{P-1}(\boldsymbol{X}))^{\mathrm{T}}$ is the vector made up of all orthonormal polynomials.

The calculation of PCE coefficients is transformed into a least-squares problem, which theoretically aims to minimize the expected error:
\begin{equation}
	\hat{\boldsymbol{y}}=\mathrm{arg} \min \mathbb{E} \left[ \left( \boldsymbol{y}^{\top}\Psi \left( \boldsymbol{X} \right) -\mathcal{M} \left( \boldsymbol{X} \right) \right) ^2 \right] 
\end{equation}

Suppose that there are $N$ input samples: $\mathcal{X}= \{\boldsymbol{x}^{(1)},\cdots,\boldsymbol{x}^{(N)}\}^{\mathrm{T}}$ and the corresponding model responses $\mathcal{Y} = \{y^{(1)},\cdots,y^{(N)}\}^{\mathrm{T}}$. The solution form given by Ordinary Least Squares (OLS) regression is as follows:
\begin{equation}
	\hat{\boldsymbol{y}}=\left( \mathbf{A}^{\top}\mathbf{A} \right) ^{-1}\mathbf{A}^{\top}\mathcal{Y} 
	\label{OLS}
\end{equation}
where $\mathbf{A}$ is the design matrix: $A_{ij} = \Psi_j(\boldsymbol{x}^{(i)})$, $i=1,\cdots,N$, $j=0,\cdots,P-1$. This matrix contains the values of all the base polynomials evaluated at each sample.

\subsection{Classical PCE-based Sobol' Indices}
Consider the truncated PCE representation of a model output Eq.~\ref{tunPCE}, the unique Sobol' decomposition characterizes the model response as a sum of functions of increasing dimensions. 
	
The PCE terms can be reorganized according to the subset of input variables they involve. For any non-empty subset $\mathbf{u} \subset \{1, \dots, M\}$, define the associated index set:
\begin{equation}
	\mathcal{A}_{\mathbf{u}} = \left\{ \boldsymbol{\alpha} \in \mathcal{A} : \alpha_k \neq 0 \iff k \in \mathbf{u} \right\}
\end{equation}
	
Using this partition, the truncated PCE can be rewritten as a sum of contributions associated with different variable subsets:
\begin{equation}
	\mathcal{M} ^{PC} = y_{\mathbf{0}} + \sum_{\substack{\mathbf{u} \subset \{1,\dots,M\} \\ \mathbf{u} \neq \emptyset}} 
	\sum_{\boldsymbol{\alpha} \in \mathcal{A}_{\mathbf{u}}} 
		y_{\boldsymbol{\alpha}} \, \Psi_{\boldsymbol{\alpha}}(\boldsymbol{\xi})
\end{equation}
	
Due to the orthonormality of the polynomial basis, the total variance of the PCE model and the partial variance associated with a subset of variables $\mathbf{u}$ can be computed directly from the expansion coefficients as:
\begin{equation}
	\mathrm{Var}[\mathcal{M} ^{PC}] = \sum_{\substack{\boldsymbol{\alpha} \in \mathcal{A} \\ \boldsymbol{\alpha} \neq \mathbf{0}}} y_{\boldsymbol{\alpha}}^2
\end{equation}
\begin{equation}
	\mathrm{Var}[\mathcal{M} ^{PC}_{\mathbf{u}}] = \sum_{\boldsymbol{\alpha} \in \mathcal{A}_{\mathbf{u}}} y_{\boldsymbol{\alpha}}^2
\end{equation}
	
The Sobol' index corresponding to $\mathbf{u}$ is therefore obtained as:
\begin{equation}
	S_{\mathbf{u}} = \frac{V_{\mathbf{u}}}{V} 
	= \frac{\sum_{\boldsymbol{\alpha} \in \mathcal{A}_{\mathbf{u}}} y_{\boldsymbol{\alpha}}^2}
	{\sum_{\substack{\boldsymbol{\alpha} \in \mathcal{A} \\ \boldsymbol{\alpha} \neq \mathbf{0}}} y_{\boldsymbol{\alpha}}^2}
\end{equation}
which enables the direct computation of Sobol' indices from PCE coefficients without additional model evaluations.

While these global indices $S_{\mathbf{u}}$ provide a comprehensive overview of variable importance, they are static and cannot reflect how sensitivities vary under specific conditions or within localized sub-domains. This limitation motivates the development of the analytical conditional framework presented in the next section.
	
\subsection{Sparse PCE via Orthogonal Matching Pursuit}
\label{OMP}
In high-dimensional settings, the full PCE basis $\mathcal{A}^{M,p}$ may lead to a large number of terms $P$, resulting in high computational cost and potential overfitting. To address this issue, sparse regression techniques are commonly employed to identify a reduced subset of significant basis functions. In this work, Orthogonal Matching Pursuit (OMP) is adopted to construct a sparse PCE representation \cite{Rezaiifar1995}.
	
Starting from Eq.~\eqref{OLS}, the model response at the sample points can be written as the following identity:
\begin{equation}
	\mathcal{Y} = \mathbf{A}\hat{\boldsymbol{y}} + \boldsymbol{R}
\end{equation}
where $\boldsymbol{R}$ denotes the residual vector associated with the OLS approximation.
	
OMP constructs a sparse approximation of the coefficient vector $\hat{\boldsymbol{y}}$ by iteratively selecting the most relevant basis functions from the candidate index set $\mathcal{J} = \{0,1,\dots,P-1\}$. At iteration $n$, let $\mathcal{J}_n \subset \mathcal{J}$ denote the active set of selected indices, and let $\mathbf{A}^{(n)}$ be the submatrix of $\mathbf{A}$ formed by the corresponding columns.
	
The residual at iteration $n$ is defined as:
\begin{equation}
	\boldsymbol{R}_n = \mathcal{Y} - \mathbf{A}^{(n)} \hat{\boldsymbol{y}}^{(n)}
\end{equation}
where $\hat{\boldsymbol{y}}^{(n)}$ is obtained solving the Ordinary Least Squares (OLS) problem constrained to the active set $\mathcal{J}_n$. At each iteration, OMP selects the index of the basis function that is most correlated with the current residual:
\begin{equation}
	j_{n+1} = \arg\max_{j \in \mathcal{J} \setminus \mathcal{J}_n} \left|\left< \mathbf{A}_{\cdot,j}, \boldsymbol{R}_n \right>\right|
\end{equation}
where $\mathcal{J} \setminus \mathcal{J}_n$ denotes the set of candidate indices not yet included in the active set, and $\mathbf{A}_{\cdot,j}$ is the $j$-th column of the design matrix $\mathbf{A}$ representing the $j$-th basis function evaluated at all sample points. The selected index is then added to the active set:
\begin{equation}
		\mathcal{J}_{n+1} = \mathcal{J}_n \cup \{j_{n+1}\}
\end{equation}

Given the design matrix $\mathbf{A} \in \mathbb{R}^{N \times P}$ and the response vector $\mathcal{Y}$, OMP iteratively selects the basis column from $\mathbf{A}$ that exhibits the maximum absolute correlation with the current residual $\boldsymbol{R}_n$. Upon selecting a new basis index, the coefficients for all currently active bases are updated via OLS to ensure that the residual is orthogonal to the span of the selected columns. To prevent overfitting, especially in cases with limited samples or spatial noise, an early stopping criterion based on the relative residual norm is implemented.

\begin{algorithm}[H]
	\caption{OMP for Sparse PCE Reconstruction}
	\label{alg:omp_final}
	\KwIn{Design matrix $\mathbf{A} \in \mathbb{R}^{N \times P}$, response vector $\mathcal{Y}$, max terms $m$, tolerance $\delta$}
	\KwOut{Sparse coefficient vector $\hat{\boldsymbol{y}}$}
	\BlankLine
	
	Initialize: $\boldsymbol{R}_0 = \mathcal{Y}$, $\mathcal{J}_0 = \emptyset$, $\hat{\boldsymbol{y}} = \mathbf{0}$, $n = 0$ ;
	
	\While{$n < m$}{
		
		$\boldsymbol{c} = \left| \mathbf{A}^\top \boldsymbol{R}_n \right|$ \tcp*{Compute correlations}
		
		$j_{n+1} = \arg\max_{j \in \mathcal{J} \setminus \mathcal{J}_n} \boldsymbol{c}_j$ ;
		
		$\mathcal{J}_{n+1} = \mathcal{J}_n \cup \{ j_{n+1} \}$ \tcp*{Update active set}
		
		$\mathbf{A}^{(n+1)} = \mathbf{A}(:, \mathcal{J}_{n+1})$ ;
		
		$\hat{\boldsymbol{y}}^{(n+1)} = (\mathbf{A}^{(n+1)\top} \mathbf{A}^{(n+1)})^{-1} \mathbf{A}^{(n+1)\top} \mathcal{Y}$ \tcp*{OLS update}
		
		$\boldsymbol{R}_{n+1} = \mathcal{Y} - \mathbf{A}^{(n+1)} \hat{\boldsymbol{y}}^{(n+1)}$ \tcp*{Update residual}
		
		\If{$\|\boldsymbol{R}_{n+1}\|_2 / \|\mathcal{Y}\|_2 < \delta$}{
			\textbf{break} ;
		}
		
		$n = n + 1$ ;
	}
	
	$\hat{\boldsymbol{y}}(\mathcal{J}_n) = \hat{\boldsymbol{y}}^{(n)}$ ;
	
	\Return{$\hat{\boldsymbol{y}}$}
\end{algorithm}

\section{Analytical conditional Sobol' analysis via basis decomposition}

In many complex systems, the model response is parameterized by a subset of input variables $\mathbf{s}$, such as spatial coordinates, temporal indices, or operational parameters. Rather than constructing independent surrogates for each discrete value of $\mathbf{s}$, a more efficient approach is to treat both the conditioning variables $\mathbf{s}$ and the random input vector $\boldsymbol{\xi}$ as components of a joint augmented input vector $\boldsymbol{X}=\left( \mathbf{s},\boldsymbol{\xi } \right)^\mathrm{T}$. By doing so, a unified PCE model can be constructed through OLS regression (as described in Eq.\ref{OLS}), providing a continuous and integrated functional representation of the entire parametric stochastic field.

\subsection{Formulation of the analytical conditional PCE}

To extract the conditional sensitivity structure with respect to $\mathbf{s}$, the multi-index $\boldsymbol{\alpha}$ is partitioned into two disjoint subsets: $\boldsymbol{\alpha}_K$, which clusters indices related to the conditioning parameters $\mathbf{s}$; and $\boldsymbol{\alpha}_L$, which clusters the indices associated with the remaining uncertain variables $\boldsymbol{\xi}$. Leveraging the tensor-product nature of the PCE basis functions, the joint basis $\Psi_{\boldsymbol{\alpha}}(\boldsymbol{X})$ can be analytically factorized into a conditional form:
\begin{equation}
	\Psi _{\boldsymbol{\alpha }}\left( \boldsymbol{X} \right) =\Psi _{\boldsymbol{\alpha }}\left( \mathbf{s},\boldsymbol{\xi } \right) =\Psi _{\boldsymbol{\alpha }_K}\left( \mathbf{s} \right) \Psi _{\boldsymbol{\alpha }_L}\left( \boldsymbol{\xi } \right)
\end{equation}
By regrouping terms based on the stochastic multi-indices $\boldsymbol{\alpha}_L$, the global PCE model is reformulated into an analytical conditional representation. This transformation reveals that the influence of the conditioning variables $\mathbf{s}$ can be entirely encapsulated within a new set of varying coefficients. The formal mathematical properties of this decomposition are summarized in the following lemma.

\begin{lemma}[Orthogonality preservation for analytical conditional PCE]
	\label{lem:ortho}
	Given the basis factorization $\Psi_{\boldsymbol{\alpha}}(\mathbf{s}, \boldsymbol{\xi}) = \Psi_{\boldsymbol{\alpha}_K}(\mathbf{s}) \Psi_{\boldsymbol{\alpha}_L}(\boldsymbol{\xi})$, the global PCE model can be analytically reformulated into a conditional expansion:
	\begin{equation}
		\mathcal{M}^{PC}(\boldsymbol{\xi} | \mathbf{s}) = \sum_{\boldsymbol{\alpha}_L} c_{\boldsymbol{\alpha}_L}(\mathbf{s}) \Psi_{\boldsymbol{\alpha}_L}(\boldsymbol{\xi})
	\end{equation}
	where the parametric coefficient fields $c{\boldsymbol{\alpha}L}(\mathbf{s})$ are defined as:
	\begin{equation}
		c_{\boldsymbol{\alpha}_L}(\mathbf{s}) = \sum_{\boldsymbol{\alpha}_K} y_{\boldsymbol{\alpha}_K, \boldsymbol{\alpha}_L} \Psi_{\boldsymbol{\alpha}_K}(\mathbf{s})
	\end{equation}
	Crucially, for any fixed value of the conditioning variables $\mathbf{s}$, the basis functions $\Psi{\boldsymbol{\alpha}_L}(\boldsymbol{\xi})$ preserve their orthonormality with respect to the random vector $\boldsymbol{\xi}$.
\end{lemma}

\begin{proof}
	Starting from the PCE formulation, we substitute the factorized basis functions $\Psi_{\boldsymbol{\alpha}}(\mathbf{s}, \boldsymbol{\xi}) = \Psi_{\boldsymbol{\alpha}_K}(\mathbf{s}) \Psi_{\boldsymbol{\alpha}_L}(\boldsymbol{\xi})$ into the expansion:
	\begin{equation}
		\mathcal{M}^{PC}(\boldsymbol{\xi} | \mathbf{s}) = \sum_{\boldsymbol{\alpha}_L} \sum_{\boldsymbol{\alpha}_K} y_{\boldsymbol{\alpha}_K, \boldsymbol{\alpha}_L} \Psi_{\boldsymbol{\alpha}_K}(\mathbf{s}) \Psi_{\boldsymbol{\alpha}_L}(\boldsymbol{\xi})
	\end{equation}
	By invoking the linearity of the summation, the terms associated with the conditioning basis functions $\Psi{\boldsymbol{\alpha}K}(\mathbf{s})$ can be grouped:
	\begin{equation}
		\mathcal{M}^{PC}(\boldsymbol{\xi} | \mathbf{s}) = \sum_{\boldsymbol{\alpha}_L} \left[ \sum_{\boldsymbol{\alpha}_K} y_{\boldsymbol{\alpha}_K, \boldsymbol{\alpha}_L} \Psi_{\boldsymbol{\alpha}_K}(\mathbf{s}) \right] \Psi_{\boldsymbol{\alpha}_L}(\boldsymbol{\xi})
	\end{equation}
	By defining the term in the brackets as the parametric coefficient $c{\boldsymbol{\alpha}_L}(\mathbf{s})$, we obtain the conditional form of the expansion.
	
	Regarding the orthogonality, consider the inner product with respect to the random vector $\boldsymbol{\xi}$ for a given $\mathbf{s}$. Since the functions $\Psi_{\boldsymbol{\alpha}_L}(\boldsymbol{\xi})$ are elements of the original orthonormal basis set $\mathcal{A}$ restricted to the stochastic dimensions of $\boldsymbol{\xi}$, their orthonormality is inherently preserved under the conditional probability measure:
	\begin{equation}
		\langle \Psi_{\boldsymbol{\alpha}_L}, \Psi_{\boldsymbol{\beta}_L} \rangle_{\boldsymbol{\xi}} = \delta_{\boldsymbol{\alpha}_L \boldsymbol{\beta}_L}
	\end{equation}
\end{proof}

A key property of this analytical conditional representation (Lemma~\ref{lem:ortho}) is that the orthonormality of the polynomial basis is preserved with respect to the stochastic variables $\boldsymbol{\xi}$. This enables the direct computation of statistical moments conditioned on any given value of $\mathbf{s}$.Specifically, the conditional mean of the model response, which represents the deterministic trend of the stochastic field varying with $\mathbf{s}$, is directly given by the zero-order coefficient field:
\begin{equation}
	\mathbb{E} \left[ \mathcal{M}^{PC}(\boldsymbol{\xi}|\mathbf{s}) \right] = c_{\boldsymbol{0}}(\mathbf{s}) = \sum_{\boldsymbol{\alpha}_K} y_{\boldsymbol{\alpha}_K, \boldsymbol{0}} \Psi_{\boldsymbol{\alpha}_K}(\mathbf{s})
\end{equation}
Subsequently, the conditional variance, which characterizes the local uncertainty magnitude at a given $\mathbf{s}$, is obtained by summing the squares of all higher-order coefficient fields:
\begin{equation}
	\mathrm{Var} \left[ \mathcal{M}^{PC}(\boldsymbol{\xi}|\mathbf{s}) \right]= \sum_{\boldsymbol{\alpha}_L \ne \boldsymbol{0}} c_{\boldsymbol{\alpha}_L}^{2}(\mathbf{s}),\label{SCvar}
\end{equation}

These expressions demonstrate that both the mean and variance are no longer static scalars, but are expressed as continuous, analytical functions of the conditioning variables $\mathbf{s}$.

\subsection{Analytical extraction of conditional Sobol' indices}

In the framework of the analytical conditional PCE, Sobol' sensitivity indices can be derived directly from the parametric coefficient fields $c_{\boldsymbol{\alpha}_L}(\mathbf{s})$, completely bypassing the need for additional model evaluations or sampling.

To obtain the conditional sensitivity distribution, the variables $\mathbf{s}$ are treated as conditioning parameters and are excluded from the functional decomposition. We focus on the sensitivity of the response with respect to the random input vector $\boldsymbol{\xi} = \{\xi_1, \dots, \xi_L\}$. For any non-empty subset of indices $\mathbf{u} \subset \{1, \dots, L\}$, we define the associated conditional multi-index set:

\begin{equation}
\mathcal{A}_{\mathbf{u}, L}^{M, p} = \{ \boldsymbol{\alpha}_L \in \mathbb{N}^L : \alpha_k \neq 0 \Longleftrightarrow k \in \mathbf{u}, \quad \text{and } (\boldsymbol{\alpha}_K, \boldsymbol{\alpha}_L) \in \mathcal{A}^{M, p}\}
\end{equation}

Based on these definitions, the analytical conditional Sobol' decomposition is formalized in the following lemma.

\begin{lemma}[Analytical conditional Sobol' decomposition]\label{lem:scsobol}
	Given the conditional PCE model $\mathcal{M}^{PC}(\boldsymbol{\xi}|\mathbf{s})$, for any fixed value of $\mathbf{s}$, there exists a unique Sobol' decomposition with respect to the stochastic input vector $\boldsymbol{\xi}$:
	\begin{equation}
		\mathcal{M} ^{PC}\left( \boldsymbol{\xi }|\mathbf{s} \right) = \mathcal{M} ^{PC}_{\boldsymbol{0}}(\mathbf{s})+\sum_{ \mathbf{u}\ne \emptyset}{\mathcal{M} ^{PC}_{\mathbf{u}}\left( \boldsymbol{\xi}_{\mathbf{u}}|\mathbf{s} \right)}
	\end{equation}
	where $\mathcal{M}^{PC}_{\boldsymbol{0}}(\mathbf{s})$ is the conditional mean, and each component $\mathcal{M}_{\mathbf{u}}$ is represented by the corresponding $\text{PCE}$ terms:
	\begin{equation}
		\mathcal{M} ^{PC} _{\mathbf{u}}\left( \boldsymbol{\xi }_{\mathbf{u}}|\mathbf{s} \right) = \sum_{\boldsymbol{\alpha }_L\in \mathcal{A} _{\mathbf{u},L}^{M,p}}{c_{\boldsymbol{\alpha }_L}(\mathbf{s})\Psi _{\boldsymbol{\alpha }_L}\left( \boldsymbol{\xi } \right)}
	\end{equation}
	
	Under the orthonormality of the polynomial chaos basis $\Psi_{\boldsymbol{\alpha}_L}$, the conditional Sobol' indices are obtained analytically from the coefficient fields as:
	
	\begin{enumerate}
		
		\item {First-order sensitivity index}
		\begin{equation}
			S_{i|\mathbf{s}} 
			= \frac{\sum_{\boldsymbol{\alpha}_L \in \mathcal{A}_{\mathbf{u},L}^{M,p}, \ \mathbf{u} = \{i\}} 
				c_{\boldsymbol{\alpha}_L}^2(\mathbf{s})}
			{\sum_{\substack{\boldsymbol{\alpha}_L \in \mathcal{A}^{M,p} \\ \boldsymbol{\alpha}_L \neq \boldsymbol{0}}} 
				c_{\boldsymbol{\alpha}_L}^2(\mathbf{s})}
		\end{equation}
		
		\item {Interaction sensitivity index}
		\begin{equation}
			S_{\mathbf{u}|\mathbf{s}} 
			= \frac{\sum_{\boldsymbol{\alpha}_L \in \mathcal{A}_{\mathbf{u},L}^{M,p}} 
				c_{\boldsymbol{\alpha}_L}^2(\mathbf{s})}
			{\sum_{\substack{\boldsymbol{\alpha}_L \in \mathcal{A}^{M,p} \\ \boldsymbol{\alpha}_L \neq \boldsymbol{0}}} 
				c_{\boldsymbol{\alpha}_L}^2(\mathbf{s})}
		\end{equation}
		
		\item {Total sensitivity index}
		\begin{equation}
			S_{i|\mathbf{s}}^{T}
			= \frac{\sum_{\substack{\boldsymbol{\alpha}_L \in \mathcal{A}^{M,p} \\ \alpha_i \neq 0}} 
				c_{\boldsymbol{\alpha}_L}^2(\mathbf{s})}
			{\sum_{\substack{\boldsymbol{\alpha}_L \in \mathcal{A}^{M,p} \\ \boldsymbol{\alpha}_L \neq \boldsymbol{0}}} 
				c_{\boldsymbol{\alpha}_L}^2(\mathbf{s})}
			\label{sensi}
		\end{equation}
	\end{enumerate}
	
\end{lemma}

\begin{proof}
	The result follows from the uniqueness of the Sobol' decomposition and the orthogonality of the basis functions $\Psi_{\boldsymbol{\alpha}_L}$ established in Lemma~\ref{lem:ortho}. By partitioning the conditional terms into the index sets $\mathcal{A}_{\mathbf{u}, L}$, the model response is decomposed into mutually orthogonal components. The conditional variance contributions are then obtained as sums of the squared parametric coefficients $c_{\boldsymbol{\alpha}_L}(\mathbf{s})$. Normalization by the total conditional variance (Eq.~\ref{SCvar}) yields the analytical sensitivity indices.
\end{proof}

\section{Numerical illustration: analytical benchmark for conditional sensitivity}
To demonstrate the capability of the proposed framework in capturing varying uncertainty, we construct a synthetic parametric stochastic model. This model serves as an analytical benchmark to verify how the conditional PCE framework decomposes total uncertainty into continuous, spatially-resolved sensitivity maps.

\subsection{Problem setup and analytical reference}
Consider a physical quantity $\mathcal{M}$ defined over a two-dimensional domain $\mathbf{s} = (x, y) \in [0, 1]^2$. The system is driven by two independent and identically distributed standard normal variables $\xi_1, \xi_2 \sim \mathcal{N}(0, 1)$. The response surface is explicitly defined to provide a ground truth for comparison:
\begin{equation}
	\mathcal{G}(\mathbf{s}, \xi_1, \xi_2) = g_0(\mathbf{s}) + g_1(\mathbf{s})\xi_1 + g_2(\mathbf{s})\xi_2 + g_{12}(\mathbf{s})\xi_1\xi_2
\end{equation}
where the parametric shaping functions are prescribed as:
\begin{equation}
	g_i(\mathbf{s}) =
	\begin{cases}
		\sin(\pi x)\cos(\pi y), & i=0,\\
		0.8\,\sin(\pi x)\cos(\pi y), & i=1,\\
		0.6\,\cos(2\pi x)\sin(\pi y), & i=2,\\
		0.4\,\sin(\pi x)\sin(\pi y), & i=12
	\end{cases}
\end{equation}

Given the independence and zero-mean property of $\xi_i$, the stochastic terms are mutually orthogonal by construction. Consequently, the analytical conditional variance $\sigma^2(\mathbf{s})$ is exactly the sum of squares of the coefficient fields:
\begin{equation}
	\sigma^2(\mathbf{s}) = \text{Var}[\mathcal{G} | \mathbf{s}] = g_1^2(\mathbf{s}) + g_2^2(\mathbf{s}) + g_{12}^2(\mathbf{s})
\end{equation}

Based on this decomposition, the conditional Sobol' indices admit closed-form analytical expressions, which will be used to validate the PCE-extracted results. The first-order and interaction indices are:
\begin{equation}
	S_{1}(\mathbf{s}) = \frac{g_1^2(\mathbf{s})}{\sigma^2(\mathbf{s})}, \quad
	S_{2}(\mathbf{s}) = \frac{g_2^2(\mathbf{s})}{\sigma^2(\mathbf{s})}, \quad
	S_{12}(\mathbf{s}) = \frac{g_{12}^2(\mathbf{s})}{\sigma^2(\mathbf{s})}
\end{equation}

The corresponding total-effect indices are defined as:
\begin{equation}
	S_{1}^{T}(\mathbf{s}) = S_1(\mathbf{s}) + S_{12}(\mathbf{s}), \quad
	S_{2}^{T}(\mathbf{s}) = S_2(\mathbf{s}) + S_{12}(\mathbf{s})
\end{equation}

As shown in Fig.~\ref{fig:sobol_analysis}, the variance decomposition is complete ($S_1 + S_2 + S_{12} = 1$), meaning the stochastic response is fully characterized by four continuous fields: the conditional variance $\sigma^2(\mathbf{s})$ (Fig.~\ref{fig:sobol:var}), the first-order indices $S_1(\mathbf{s})$ and $S_2(\mathbf{s})$ (Figs.~\ref{fig:sobol:s1} and \ref{fig:sobol:s2}), and the interaction index $S_{12}(\mathbf{s})$ (Fig.~\ref{fig:sobol:s12}). These analytical maps provide the exact reference to assess the accuracy of the proposed analytical extraction method.

\begin{figure}[h]
	\centering
	\begin{subfigure}[b]{0.45\textwidth}
		\centering
		\includegraphics{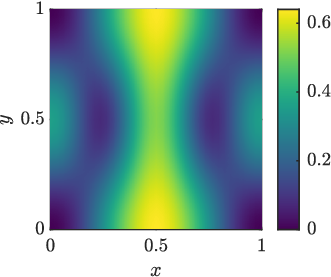}
		\caption{Total variance field $\sigma^2(\mathbf{s})$}
		\label{fig:sobol:var}
	\end{subfigure}
	\hfill
	\begin{subfigure}[b]{0.45\textwidth}
		\centering
		\includegraphics{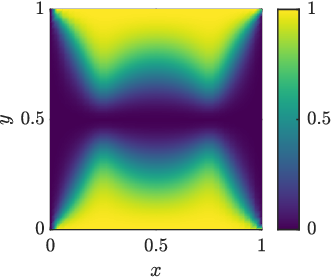}
		\caption{First-order index $S_1(\mathbf{s})$}
		\label{fig:sobol:s1}
	\end{subfigure}
	\vspace{1em}
	\begin{subfigure}[b]{0.45\textwidth}
		\centering
		\includegraphics{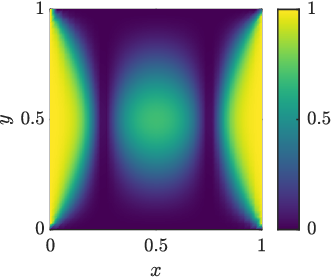}
		\caption{First-order index $S_2(\mathbf{s})$}
		\label{fig:sobol:s2}
	\end{subfigure}
	\hfill
	\begin{subfigure}[b]{0.45\textwidth}
		\centering
		\includegraphics{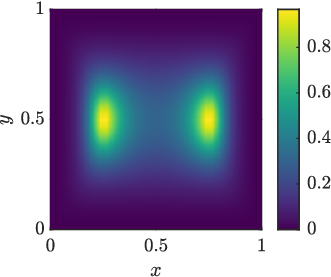}
		\caption{Interaction index $S_{12}(\mathbf{s})$}
		\label{fig:sobol:s12}
	\end{subfigure}
	\caption{Analytical spatially correlated variance and Sobol' sensitivity indices}
	\label{fig:sobol_analysis}
\end{figure}

\subsection{Performance benchmarking and error analysis}
To evaluate the performance of the proposed analytical conditional framework, its computational efficiency and accuracy are benchmarked against two representative strategies: point-wise Monte Carlo (MC) estimation \cite{sudret2007uncertainty} and point-wise PCE-based analysis \cite{salehi2018flow}. The physical domain is discretized into a $35 \times 35$ grid, resulting in $1{,}225$ spatial evaluation points.

For a fair comparison, all methods utilize an identical sample size of $N=500$, ensuring that performance differences arise solely from the surrogate modeling and sensitivity propagation strategies rather than the raw data volume.

To mimic realistic high-fidelity simulation outputs (e.g., CFD), a spatially varying Gaussian noise field $\varepsilon(\mathbf{s})$ is introduced into the analytical response:
\begin{equation}
	\mathcal{G}_{\text{noisy}}(\mathbf{s}, \boldsymbol{\xi})
	= \mathcal{G}(\mathbf{s}, \boldsymbol{\xi})
	+ \varepsilon(\mathbf{s}),
\end{equation}
where the noise magnitude is scaled by the local standard deviation and a Gaussian decay kernel:
\begin{equation}
	\varepsilon(\mathbf{s})
	= \sigma \cdot \mathrm{std}(\mathcal{G})
	\cdot \mathcal{N}(0,1)
	\cdot \exp\left(-\frac{(x-0.5)^2 + (y-0.5)^2}{2\sigma_n^2}\right).
\end{equation}

The admissible polynomial order is constrained by the available samples $N$ such that $P = \binom{M+p}{p} \leq N$. For the point-wise PCE formulation (where the spatial variables are fixed, thus $M=2$), this allows for a high order of $p\leq 30$. In contrast, for the proposed joint spatio-parametric formulation ($M=4$), the maximum order is truncated at $p\leq 8$ to manage the combinatorial growth of the basis set.

Despite the lower polynomial order in the joint space, the proposed method offers distinct computational advantages. While the point-wise PCE requires the construction of $1{,}225$ independent surrogates (one for each grid point), the proposed method constructs a single, unified model over the augmented space $\boldsymbol{X} = (\mathbf{s}, \boldsymbol{\xi})^\mathrm{T}$.

\subsection{Results and discussion}

The spatial sensitivity fields and the corresponding approximation errors $\epsilon(\mathbf{s}) = |S_{12}(\mathbf{s}) - S_{12}^{\text{true}}(\mathbf{s})|$ are illustrated in Fig.~\ref{fig:method_comparison}.

As observed, both the MC and point-wise PCE approaches yield sensitivity maps characterized by pronounced spatially uncorrelated fluctuations. This noise arises because each spatial location is treated as an isolated entity; the methods lack a mechanism to enforce the underlying physical continuity of the field. Consequently, the resulting $S_{12}$ fields are not strictly continuous, and the error distributions appear scattered.

In contrast, the proposed analytical conditional PCE produces smooth, spatially coherent sensitivity maps. By embedding the spatial coordinates $\mathbf{s}$ directly into the orthonormal basis functions $\Psi_{\boldsymbol{\alpha}_K}(\mathbf{s})$, the framework inherently applies a spatial regularization effect. This ensures that the extracted Sobol' indices inherit the analytical continuity of the basis functions, effectively filtering out pointwise numerical noise and providing a more robust characterization of the stochastic field.

In terms of quantitative accuracy under identical sample sizes, the MC-based reconstruction exhibits the largest discrepancies, with the error field $\epsilon(\mathbf{s})$ fluctuating significantly within $[0, 1.5 \times 10^{-1}]$. While the point-wise PCE method successfully reduces the peak error magnitude to approximately $1.1 \times 10^{-2}$, the resulting error field remains spatially unstructured. This lack of coherence stems from the independent construction of surrogate models at each spatial location, which fails to filter out pointwise numerical noise.

In contrast, the proposed analytical conditional PCE yields smooth and physically consistent sensitivity fields by explicitly exploiting the joint spatio-parametric structure of the response. The reconstructed $S_{12}$ field successfully preserves the inherent spatial coherence of the analytical solution.

Although the proposed method introduces a minor bias ($\epsilon(\mathbf{s}) \approx 1.75 \times 10^{-2}$) due to the global regression process, this error is spatially structured rather than random. This characteristic indicates that the framework acts as a global filter, trading off a marginal increase in local peak error for a significant gain in overall field stability and physical consistency. Ultimately, the proposed framework enhances robustness against sampling noise and ensures that the sensitivity analysis remains stable across the entire parametric domain, making it more suitable for high-fidelity engineering applications like CFD.

\begin{figure}[htp]
	\centering
	\begin{subfigure}[b]{0.45\textwidth}
		\centering
		\includegraphics{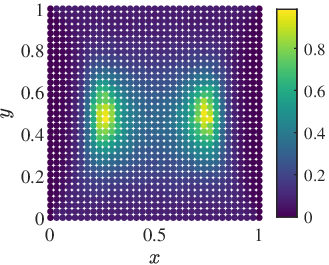}
		\caption{$S_{12}$ (MC)}
		\label{fig:mc:s12}
	\end{subfigure}
	\hfill
	\begin{subfigure}[b]{0.45\textwidth}
		\centering
		\includegraphics{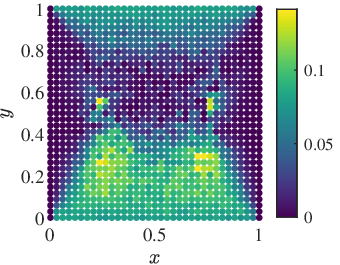}
		\caption{Error $\epsilon$ (MC)}
		\label{fig:mc:err}
	\end{subfigure}
	\vspace{1em}
	\begin{subfigure}[b]{0.45\textwidth}
		\centering
		\includegraphics{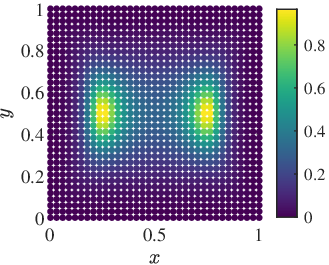}
		\caption{$S_{12}$ (Point-wise PCE)}
		\label{fig:pwpce:s12}
	\end{subfigure}
	\hfill
	\begin{subfigure}[b]{0.45\textwidth}
		\centering
		\includegraphics{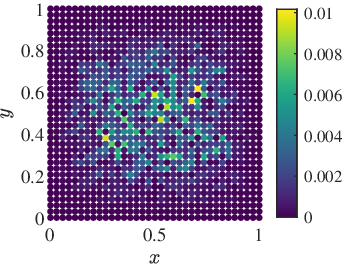}
		\caption{Error $\epsilon$ (Point-wise PCE)}
		\label{fig:pwpce:err}
	\end{subfigure}
	\vspace{1em}
	\begin{subfigure}[b]{0.45\textwidth}
		\centering
		\includegraphics{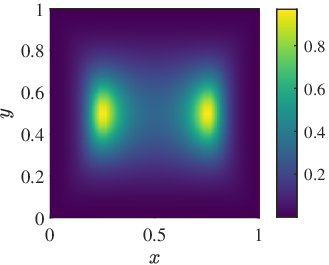}
		\caption{$S_{12}$ (SCPCE)}
		\label{fig:scpce:s12}
	\end{subfigure}
	\hfill
	\begin{subfigure}[b]{0.45\textwidth}
		\centering
		\includegraphics{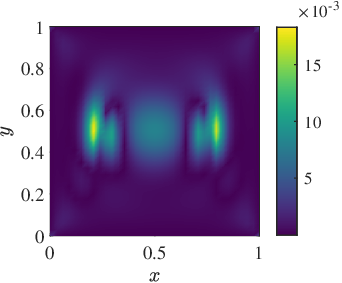}
		\caption{Error $\epsilon$ (SCPCE)}
		\label{fig:scpce:err}
	\end{subfigure}
	
	\caption{Comparison of spatial distributions of the interaction index $S_{12}$ and the corresponding error fields for different methods.}
	\label{fig:method_comparison}
\end{figure}

\section{Conclusion and Future Work}

\subsection{Conclusion}
In this paper, we have presented a unified framework for the analytical extraction of conditional Sobol' indices from a global Polynomial Chaos Expansion (PCE). By treating the spatial coordinates (or more generally, any conditioning parameters) $\mathbf{s}$ and the stochastic variables $\boldsymbol{\xi}$ as a joint augmented input vector, we have derived a series of analytical lemmas that allow for the direct computation of spatially-varying sensitivity maps.

The key contributions and findings are summarized as follows:

\begin{itemize}
\item We proved that through basis factorization, the influence of conditioning variables can be entirely encapsulated within parametric coefficient fields $c_{\boldsymbol{\alpha}_L}(\mathbf{s})$, enabling the extraction of conditional moments and Sobol' indices without the need for repetitive point-wise modeling or sampling.

\item Unlike traditional point-wise approaches (MC or PCE) that yield noisy, uncorrelated sensitivity fields, the proposed method inherits the analytical continuity of the underlying basis functions. This provides a "global filtering" effect that ensures physical consistency and robustness against simulation noise.

\item The proposed method constructs a single unified surrogate, significantly reducing the overhead associated with managing thousands of independent local models, while maintaining high quantitative accuracy.
\end{itemize}

\subsection{Future Work}
Despite the advantages of the current framework, several challenges remain for large-scale engineering applications:

\begin{itemize}
\item When the number of spatial discrete points $\mathbf{s}$ is extremely large (e.g., millions of grid cells in high-fidelity CFD), the augmented regression becomes computationally intensive. Future research will focus on dimensionality reduction of the spatial basis or the use of domain decomposition to alleviate the memory burden during coefficient estimation.

\item Current sparse regression techniques (such as OMP or LASSO) lack robust, spatially-aware error metrics to guide the selection of joint spatio-parametric terms. Developing localized error estimators that can adaptively trigger basis enrichment would be crucial for accelerating the convergence of high-dimensional conditional models.

\item Extending the current analytical framework to handle non-stationary or non-Gaussian conditioning fields will further broaden its applicability to multi-scale and multi-physics problems.
\end{itemize}

\bibliographystyle{unsrt}  
\bibliography{bib}  

\end{document}